\documentclass[11pt,draftcls,onecolumn]{IEEEtran}
\usepackage{amssymb}
\usepackage{multirow,setspace,verbatim,amsfonts,graphicx,amsmath,amsthm,amsbsy,amssymb,epsfig,url}
\usepackage{soul}
\soulregister\cite7 
\soulregister\ref7
\usepackage{bm} 

\makeatletter
\newcommand\xleftrightarrow[2][]{%
  \ext@arrow 9999{\longleftrightarrowfill@}{#1}{#2}}
\newcommand\longleftrightarrowfill@{%
  \arrowfill@\leftarrow\relbar\rightarrow}
\makeatother

\newcommand{\Hb}{{\mathbf H}}
\newcommand{\Ib}{{\mathbf I}}

\newcommand{\Wb}{{\mathbf W}}

\newcommand{\ab}{{\mathbf a}}

\newcommand{\xb}{{\mathbf x}}
\newcommand{\yb}{{\mathbf y}}

\newcommand{\Rd}{{\mathbb R}}

\DeclareMathOperator*{\prox}{prox}
\DeclareMathOperator*{\argmin}{arg\,min}
\DeclareMathOperator*{\eig}{eig}

\usepackage{cite}
\usepackage{makecell}

\ifCLASSINFOpdf

\else

\fi

\usepackage{amsfonts}
\usepackage{amsmath,bm}
\usepackage{amssymb}
\usepackage{amsthm}
\usepackage{tikz,graphicx}
\usepackage{mathrsfs} 
\usepackage{algpseudocode}
\usepackage{algorithm}
\usepackage{array,multirow}
\usepackage{color}
\usepackage{diagbox}

\newtheorem{definition}{Definition}
\newtheorem{theorem}{Theorem}

\newcommand{\MONTH}{%
  \ifcase\the\month
  \or JANUARY
  \or FEBRUARY
  \or MARCH
  \or APRIL
  \or MAY
  \or JUNE
  \or JULY
  \or AUGUST
  \or SEPTEMBER
  \or OCTOBER
  \or NOVEMBER
  \or DECEMBER
  \fi}

\begin{document}
\title{Deep Convolutional Neural Network for Inverse Problems in Imaging}

\author{Kyong~Hwan~Jin,~
Michael~T.~McCann,~\IEEEmembership{Member,~IEEE},
Emmanuel~Froustey,
Michael~Unser,~\IEEEmembership{Fellow,~IEEE}
\thanks{K.H. Jin acknowledges the support from the ``EPFL Fellows'' fellowship program co-funded by Marie Curie from the European Union’s Horizon 2020 Framework Programme for Research and Innovation under grant agreement 665667.}
\thanks{K.H. Jin is with the Biomedical Imaging Group, EPFL, Lausanne, Switzerland (e-mail:kyonghwan.jin@gmail.com).}
\thanks{Michael McCann is with the Center for Biomedical Imaging, Signal Processing Core and the Biomedical Imaging Group, EPFL, Lausanne, Switzerland (e-mail:michael.mccann@epfl.ch).}
\thanks{E. Froustey is with Dassault Aviation, Saint-Cloud, France, previously with the Biomedical Imaging Group, EPFL, Lausanne, Switzerland. (e-mail:emmanuel.froustey@dassault-aviation.com).
}
\thanks{Michael Unser is with the Biomedical Imaging Group, EPFL, Lausanne, Switzerland (e-mail:michael.unser@epfl.ch).}%
}

\markboth{}%
{}

\maketitle

\begin{abstract}
\baselineskip 0.17in
In this paper, we propose a novel deep convolutional neural network (CNN)-based algorithm for solving ill-posed inverse problems.
Regularized iterative algorithms have emerged as the standard approach to ill-posed inverse problems in the past few decades.
These methods produce excellent results, but can be challenging to deploy in practice due to factors including the high computational cost of the forward and adjoint operators and the difficulty of hyper parameter selection.
The starting point of our work is the observation that unrolled iterative methods have the form of a CNN (filtering followed by point-wise non-linearity) when the normal operator ($H^*H$, the adjoint of $H$ times $H$) of the forward model is a convolution.
Based on this observation, we propose using direct inversion followed by a CNN to solve normal-convolutional inverse problems.
The direct inversion encapsulates the physical model of the system, but leads to artifacts when the problem is ill-posed;
the CNN combines multiresolution decomposition and residual learning in order to learn to remove these artifacts while preserving image structure.
 We demonstrate the performance of the proposed network in sparse-view reconstruction (down to 50 views) on parallel beam X-ray computed tomography in synthetic phantoms as well as in real experimental sinograms.
 The proposed network outperforms total variation-regularized iterative reconstruction for the more realistic phantoms and requires less than a second to reconstruct a $512\times512$ image on the GPU.
\end{abstract}


%
\IEEEpeerreviewmaketitle

\section{Introduction}
\PARstart{O}{ver} the past decades, iterative reconstruction methods have become the dominant approach to solving inverse problems in imaging including denoising \cite{rudin1992nonlinear,chambolle2004algorithm,chang2000adaptive,luisier2007new}, deconvolution \cite{chan1998total,krishnan2009fast}, and interpolation \cite{thevenaz2000interpolation,bertalmio2000image}.
Thanks to robust regularizers such as total variation \cite{rudin1992nonlinear,chan1998total,chambolle2004algorithm} and sparsity \cite{tibshirani1996regression}, practical algorithms have appeared with excellent image quality and reasonable computational complexity.
These advances have been particularly influential in the field of biomedical imaging, e.g., in  magnetic resonance imaging (MRI)~\cite{jung_k-t_2009,lingala_accelerated_2011} and X-ray computed tomography (CT)~\cite{mcgaffin_alternating_2015, McCann:16}. These devices face an unfavorable trade-off between noise and acquisition time. Short acquisitions lead to severe degradations of image quality, while long acquisitions may cause motion artifacts, patient discomfort, or even patient harm in the case of radiation-based modalities.
Iterative reconstruction with regularization provides a way to mitigate these problems in software, i.e. without developing new scanners.
With the appearance of compressed sensing \cite{candes2006robust}, our theoretical understanding of these approaches evolved further, and, in some applications, remarkable outcomes appeared with stable reproducibility \cite{lustig2007sparse,sidky2008image}.

A more recent trend is deep learning~\cite{lecun2015deep}, which has arisen as a promising framework providing state-of-the-art performance for image classification \cite{krizhevsky2012imagenet,russakovsky2015imagenet} and segmentation \cite{girshick2014rich,long2015fully,ronneberger2015u}. Moreover, regression-type neural networks demonstrated impressive results on inverse problems with exact models such as signal denoising\cite{burger2012image,xie2012image}, deconvolution\cite{xu2014deep}, and interpolation \cite{dong2016image,kim2015accurate}.
Central to this resurgence of neural networks has been the convolutional neural network (CNN) architecture.
Whereas the classic multilayer perceptron consists of layers that can perform arbitrary matrix multiplications on their input, the layers of a CNN are restricted to perform convolutions, greatly reducing the number of parameters which must be learned.

Researchers have begun to investigate the link between conventional approaches and deep learning networks \cite{gregor2010learning,chen2015learning,xin2016maximal,lin2016does}. Gregor and LeCun~\cite{gregor2010learning} explored the similarity between the ISTA algorithm~\cite{daubechies2004iterative} and a shared layerwise neural network and demonstrated that several layer-wise neural networks act as a fast approximated sparse coder. In \cite{chen2015learning}, a nonlinear diffusion reaction process based on the Perona-Malik process was proposed using deep convolutional learning; convolutional filters from diffusion terms were trained instead of using well-chosen filters like kernels for diffusion gradients, while the reaction terms were matched to the gradients of a data fidelity term. In \cite{xin2016maximal}, the authors focused on the relationship between $l_0$ penalized-least-squares methods and deep neural networks. In the context of a clustered dictionary  model, they found that the non-shared layer-wise independent weights and activations of a deep neural network provide more performance gain than the layer-wise fixed parameters of an unfolded $l_0$ iterative hard thresholding method. The quantitative analysis relied on the restricted isometry property (RIP) condition from compressed sensing \cite{candes2006robust}. Others have investigated learning optimal shrinkage operators for deep-layered neural networks \cite{kamilov2016learning,schmidt2014shrinkage}. 

Despite these works, practical and theoretical questions remain regarding the link between iterative reconstruction and CNNs.
For example,
in which problems can CNNs outperform traditional iterative reconstructions, and why? Where does this performance come from, and can the same gains be realized by learning aspects of the iterative process (e.g. the shrinkage)?
Although \cite{chen2015learning} began to address this connection, they only assumed that the filters learned in the Perona-Malik scheme are modified gradient kernels, with performance gains coming from the increased size of the filters.

In this paper, we explore the relationship between CNNs and iterative optimization methods for one specific class of inverse problems: those where the normal operator associated with the forward model ($H^*H$, where $H$ is the forward operator and $H^*$ is the adjoint operator) is a convolution. The class trivially includes denoising and deconvolution, but also includes MRI \cite{pruessmann1999sense}, X-ray CT \cite{sidky2008image,McCann:16}, and diffraction tomography (DT). Based on this connection, we propose a method for solving these inverse problems by combining a fast, approximate solver with a CNN. We demonstrate the approach on low-view CT reconstruction, using filtered back projection (FBP) and a CNN that makes use of residual learning \cite{kim2015accurate} and multilevel learning \cite{ronneberger2015u}. We use high-view FBP reconstructions for training, meaning that training is possible from real data (without oracle knowledge). We compare to a state-of-the art regularized iterative reconstruction and show promising results on both synthetic and real CT data. Especially, reconstructed images from the proposed network represented well complex textures which are important to diagnosis.

\section{Inverse Problems with Shift-Invariant Normal Operators}
\label{sec:inv}

We begin our discussion by describing the class of inverse problems for which the normal operator is a convolution.
As we will show, solving problems of this form iteratively requires repeated convolutions and point-wise nonlinearities, which suggests that CNNs may offer an alternative solution.
The class is broad, encompassing at least denoising, deconvolution, and reconstruction of MRI, CT, and diffraction tomography images.
The underlying convolutional structure is known for MRI and CT and has been exploited in the past for the design of fast algorithms (e.g. \cite{McCann:16}).
Here, we aim to give a general and concise definition to motivate our method.
We go on to discuss the standard direct and iterative approaches to solving these problems. 

\subsection{Theory}
For the continuous case, let $H : L_2(\mathbb{R}^{d_1}) \to L_2(\Omega)$ be a linear operator and $H^*$ denote its adjoint, where $L_2(\Omega) = \{f : \Omega \to \mathbb{C} \mid \int_\Omega |f(x)|^2 dx < +\infty \}$.
The range, $\Omega \subseteq \mathbb{R}^{d_2}$, remains general to include operators such as the X-ray transform, where the measurements are defined on a circular/spherical domain.
The following definitions give the building blocks of a normal shift-invariant operator. 

\begin{definition}[Multiplication]
\label{def:mult}
A \emph{multiplication}, $M_m : L_2(\Omega ) \to L_2(\Omega )$, is a linear operator such that $M_m \{f\}(x) = m(x) f(x)$ with $m \in L_2(\Omega)$ for some continuous, bounded function, $m: \Omega \to \mathbb{C}$.
\end{definition}

\begin{definition}[Convolution]
\label{def:conv}
A \emph{convolution}, $H_h: L_2(\Omega ) \to L_2(\Omega )$, is a linear operator such that $H_h f = \mathcal{F}^* M_{\hat{h}} \mathcal{F}f$, where $\mathcal{F}$ is the Fourier transform, $\hat{h}$ is the Fourier transform of $h$, and $M_{\hat{h}}$ is a multiplication.
\end{definition}

\begin{definition}[Reversible change of variables]
\label{def:cofv}
A \emph{reversible change of variables}, $\Phi_\varphi: L_2(\Omega_1 ) \to L_2(\Omega_2)$, is a linear operator such that $\Phi_\varphi f = f(\varphi(\cdot))$ for some $\phi:\Omega_2 \to \Omega_1 $ and such that its inverse, $\Phi_\varphi^{-1} = \Phi_{\varphi^{-1}}$ exists.
\end{definition}

If $H_h$ is a convolution, then $H_h^*H_h$ is as well (because $\mathcal{F}^* M_{\hat{h}}^* \mathcal{F} \mathcal{F}^* M_{\hat{h}} \mathcal{F} = \mathcal{F}^* M_{|\hat{h}|^2} \mathcal{F}$), but this is true for a wider set of operators.
Theorem~\ref{thm:cont} describes this set.

\begin{theorem}[Normal-convolutional operators]
\label{thm:cont}
If there exists an isometry, $T$, a multiplication, $M_m$, and a change of variables, $\Phi_\varphi$, such that $H = T M_m \Phi_{\varphi}^{-1} \mathcal{F}$, then $H^*H$ is a convolution with $\hat{h} = |\det J_\varphi| M_{\Phi_\varphi |m|^2} $, where $J_\varphi$ is the Jacobian matrix of $\varphi$ and $M_{\Phi_\varphi |m|^2}$ is a suitable multiplication.
\end{theorem}
\begin{proof}
Given an operator, $H$, that satisfies the conditions of Theorem~\ref{thm:cont}, 
\begin{align}
    H^*H 
        &= \mathcal{F}^* (\Phi_{\varphi}^{-1})^* M_m^* T^* T M_m \Phi_{\varphi}^{-1} \mathcal{F}\\
        &\overset{(a)}{=} \mathcal{F}^* (\Phi_\varphi^{-1})^* M_{|m|^2} \Phi_\varphi^{-1} \mathcal{F}\\
        &\overset{(b)}{=} \mathcal{F}^* |\det J_\varphi| M_{\Phi_\varphi |m|^2}   \mathcal{F}\label{eq:thm3}
\end{align}
 %
where (a) follows from the definitions of isometry and multiplication
and (b) follows from the definition of a reversible change of variables.
Thus, $H^*H$ is a convolution by Definition~\ref{def:conv}.
\end{proof}

A version of Theorem~\ref{thm:cont} also holds in the discrete case; we sketch the result here.
Starting with a continuous-domain operator, $H_c$, that satisfies the conditions of Theorem~\ref{thm:cont}, we form a discrete-domain operator, $H_d : l_2(\mathbb{Z}^{d_0}) \to l_2(\mathbb{Z}^{d_1}), H = S H_c Q$, where $S$ and $Q$ are sampling and interpolation, respectively.
Then, assuming that $H_c Q f$ is bandlimited, $H_d^*H_d$ is a convolution.

For example, consider the continuous 2D X-ray transform, $R : L_2(\mathbb{R}^2) \to L_2([0, \pi) \times \mathbb{R})$, which measures every line integral of a function of 2D space, indexed by the slope and intercept of the line.
Using the Fourier central slice theorem~\cite{kak_principles_2001}, 
\begin{equation}
    R = T \Phi_{\varphi}^{-1} \mathcal{F},
\end{equation}
where $\Phi_{\varphi}$ changes from Cartesian to polar coordinates (i.e.  $\varphi^{-1}(\theta, r) =(r \cos \theta, r \sin \theta)$) and $T$ is the inverse Fourier transform with respect to $r$ (which is an isometry due to Parseval's theorem).
This maps a function, $f$, of space, $\bm{x}$, to its Fourier transform, $\hat{f}$, which is a function of frequency, $\bm{\omega}$.
Then, it performs a change of variables, giving $\hat{f}_{\text{polar}}$, which is a function of a polar frequency variables, $(\theta, r)$.
Finally, $T$ inverts the Fourier transform along $r$, resulting in a sinogram that is a function of $\theta$ and a polar space variable, $y$.
Theorem~\ref{thm:cont} states that $R^*R$ is a convolution with $\hat{h}(\bm{\omega}) = |\det J_\varphi(\bm{\omega})| = 1/\|\bm{\omega}\|$, where, again, $\bm{\omega}$ is the frequency variable associated with the 2D Fourier transform, $\mathcal{F}$.

\subsection{Direct Inversion}\label{sec:direct_inv}
Given a normal-convolutional operator, $H$, the inverse (or reconstruction) problem is to recover an image $f$ from its measurements $g = Hf$.
The theory presented above suggests two methods of direct solutions to this problem.
The first is to apply the inverse of the filter corresponding to $H^*H$ to the back projected measurements, 
\begin{equation*}
    f =  W_h H^* g,
\end{equation*}
where $W_h$ is a convolution operator with $\hat{h}(\bm{\omega}) =  1 / (\lvert \det J_\varphi \rvert  \Phi_\varphi   \lvert m(\bm{\omega})\rvert^2)$.
This is exactly equivalent to performing a deconvolution in the reconstruction space.
The second is to invert the action of $H$ in the measurement domain before back projecting,
\begin{equation*}
f = H^* T M_h T^*g,
\end{equation*}
where $M_h$ is a multiplication operator with $h(\bm{\omega}) = 1/ (|\det J_{\varphi}| |m(\bm{\omega})|^2)$.
If $T$ is a Fourier transform, then this inverse is a filtering operation followed by a back projection;
if $T$ is not, the operation remains filtering-like in the sense that it is diagonalizable in the transform domain associated with $T$.
Note also that if $T$ is not a Fourier transform, then the variable $\bm{\omega}$  no longer refers to frequency.
Given the their filter-like form, we refer to these direct inverses as filtered back projection (FBP)~\cite{kak_principles_2001}, a term borrowed from X-ray CT reconstruction.

Returning to the example of the continuous 2D X-ray transform, the first method would be to back project the measurements and then apply the filter with a 2D Fourier transform given by $\|\bm{\omega}\|$.
The second approach would be to apply the filter with 1D Fourier transform given by $\omega$ to each angular measurement and then back project the result.
In the continuous case, the methods are equivalent, but, in practice, the measurements are discrete and applying these involves some approximation.
Then, which form is used affects the accuracy of the reconstruction (along with the runtime).
This type of error can be mitigated by formulating the FBP to explicitly include the effects of sampling and interpolation (e.g., as in \cite{horbelt_discretization_2002}).
The larger problem is that the filter greatly amplifies noise, thus in practice some amount of smoothing is also applied.



\subsection{Iterative Inversion}\label{sec:iter_inv}
In practice, inverse problems related with imaging are often ill-posed, which prohibits the use of direct inversion because measurement noise causes serve perturbations in the solution. Adding regularization (e.g., total variation \cite{chambolle2004algorithm} or $l_1$ sparsity as in LASSO \cite{tibshirani1996regression}) overcomes this problem.
We now adopt the discrete, finite-support notation where the forward model is a matrix, $\Hb \in \mathbb{R}^{N_\yb \times N_\xb}$ and the measurements are a vector, $\yb \in \mathbb{R}^{N_\yb}$.
The typical synthesis form of the inverse problem is then
\begin{eqnarray}\label{eq:cost2}
	\argmin_\ab \|\yb-\Hb \Wb \ab \|^2_2+\lambda\|\ab\|_1,
\end{eqnarray}
where $\ab \in \mathbb{R}^{N_\ab}$ is the vector of transform coefficients of the reconstruction such that $\xb = \Wb \ab$ is the desired reconstruction and where $\Wb\in \Rd^{N_\xb \times N_a}$ is a transform so that $\ab$ is sparse.
For example, if $\Wb$ is a multichannel wavelet transform $\Wb=\begin{bmatrix} \Wb_1 & \Wb_2 & \cdots & \Wb_c \end{bmatrix}$\cite{da2006nonsubsampled,mallat1999wavelet}, then the formulation promotes the wavelet-domain sparsity of the solution.
And, for many such transforms, $\Wb$ will be shift-invariant (a set of convolutions).

This formulation does not admit a closed form solution, and, therefore, is typically solved iteratively.
For example, the popular ISTA~\cite{daubechies2004iterative,gregor2010learning} algorithm solves Eq. \eqref{eq:cost2} with the iterate
\begin{eqnarray}
\ab^{k+1} = \prox_{l_1}(\ab^{k},\lambda) &\triangleq&
\mathcal{S}_{\lambda/L}\left( \frac{1}{L} \Wb^* \Hb^* \yb +(\Ib-\frac{1}{L} \Wb^* \Hb^* \Hb\Wb) \ab^{k} \right)\label{eq:prox3}
\end{eqnarray}
where $\mathcal{S}_{\theta}$ is the soft-thresholding operator by value $\theta$ and $L\leq\eig(\Wb^*\Hb^*\Hb\Wb)$ is the Lipschitz constant of a normal operator.
When the forward model is normal-convolutional and when $\Wb$ is a convolution, the algorithm consists of iteratively filtering by $\Ib-(1/L) \Wb^*\Hb^*\Hb\Wb$, adding a bias, $(1/L) \Wb^*\Hb^*\yb$, and applying a point-wise nonlinearity, $\mathcal{S}_{\theta}$.
This is illustrated as a block diagram with unfolded iterates in Fig. \ref{fig:block_diagram} (b).
Many other iterate methods for solving Eq. \eqref{eq:cost2}, including ADMM \cite{boyd2011distributed}, FISTA\cite{beck2009fast}, and SALSA\cite{afonso2010fast}, also rely on these basic building blocks.


\begin{figure*}[!h]
\centering
\includegraphics[trim = 0mm 0mm 0mm 0mm,clip=true,width=15 cm]{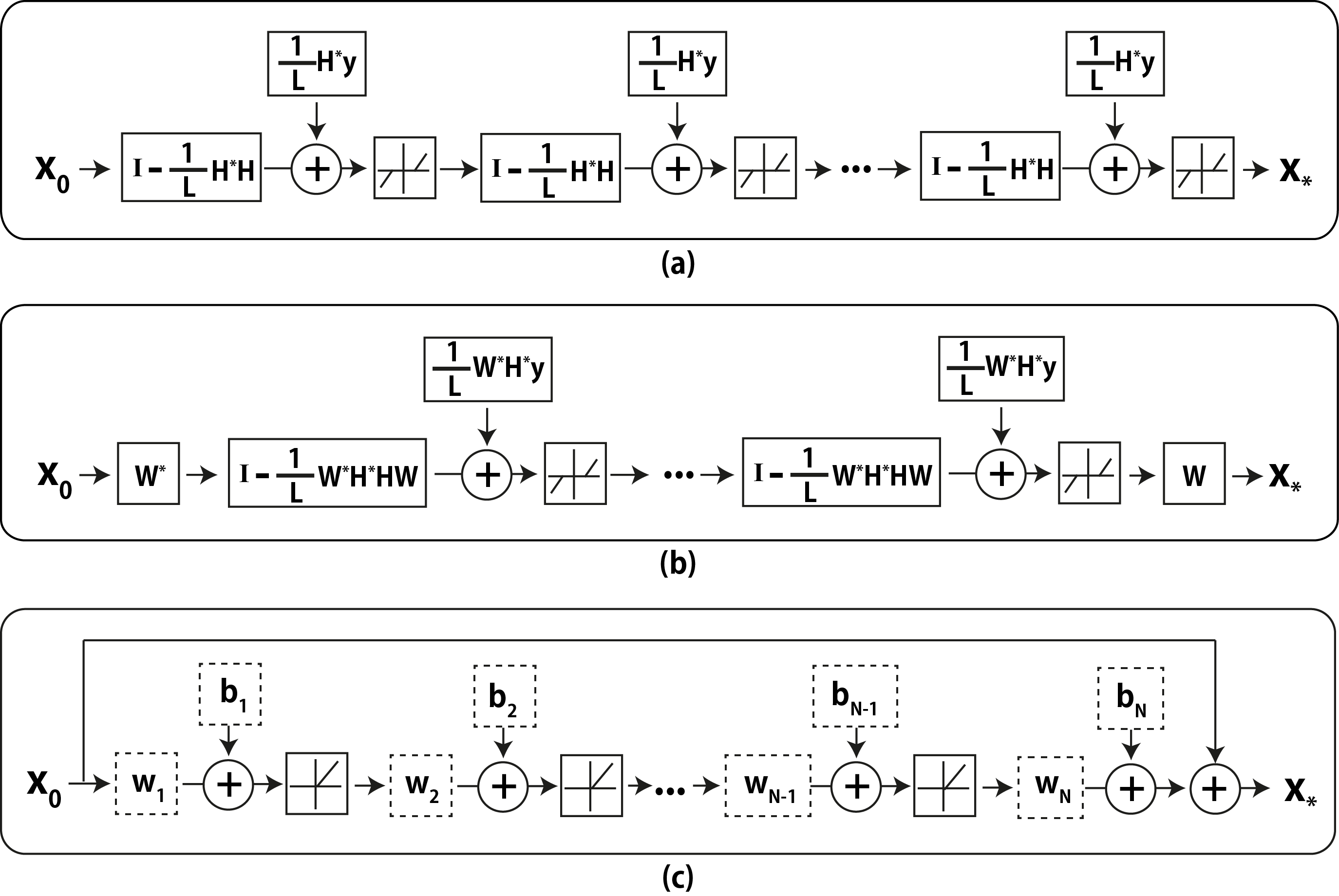}
\caption{Block diagrams about (a) unfolded version of iterative shrinkage method \cite{gregor2010learning}, (b) unfolded version of iterative shrinkage method with sparsifying transform ($\Wb$) and (c) convolutional network with the residual framework. $L$ is the Lipschitz constant, $\xb_0$ is the initial estimates, $b_i$ is learned bias, $w_i$ is learned convolutional kernel. The broken line boxes in (c) indicate the variables to be learned. }
\label{fig:block_diagram}
\end{figure*}

\section{Proposed Method: FBPConvNet}
\label{sec:proposed}
The success of iterative methods consisting of filtering plus pointwise nonlinearities on normal-convolutional inverse problems suggests that CNNs may be a good fit for these problems as well.
Based on this insight, we propose a new approach to these problems, which we call the \emph{FBPConvNet}.
The basic structure of the FBPConvNet algorithm is to apply the discretized FBP from Section~\ref{sec:direct_inv} to the measurements and then use this as the input of a CNN which is trained to regress the FBP result to a suitable ground truth image.
This approach applies in principle to all normal-convolutional inverse problems, but we have focused in this work on CT reconstruction.
We now describe the method in detail.

\subsection{Filtered Back Projection}
While it would be possible to train a CNN to regress directly from the measurement domain to the reconstruction domain, performing the FBP first greatly simplifies the learning.
The FBP encapsulates our knowledge about the physics of the inverse problem and also provides a warm start to the CNN.
For example, in the case of CT reconstruction, if the sinogram is used as input, the CNN must encode a change between polar and Cartesian coordinates, which is completely avoided when the FBP is used as input.
We stress again that, while the FBP is specific to CT, Section \ref{sec:iter_inv} shows that efficient, direct inversions are always available for normal-convolutional inverse problems.

\subsection{Deep Convolutional Neural Network Design}\label{sec:CNN_design}

\begin{figure*}[htbp]
\centering
\includegraphics[trim = 0mm 0mm 0mm 0mm,clip=true,width=14cm]{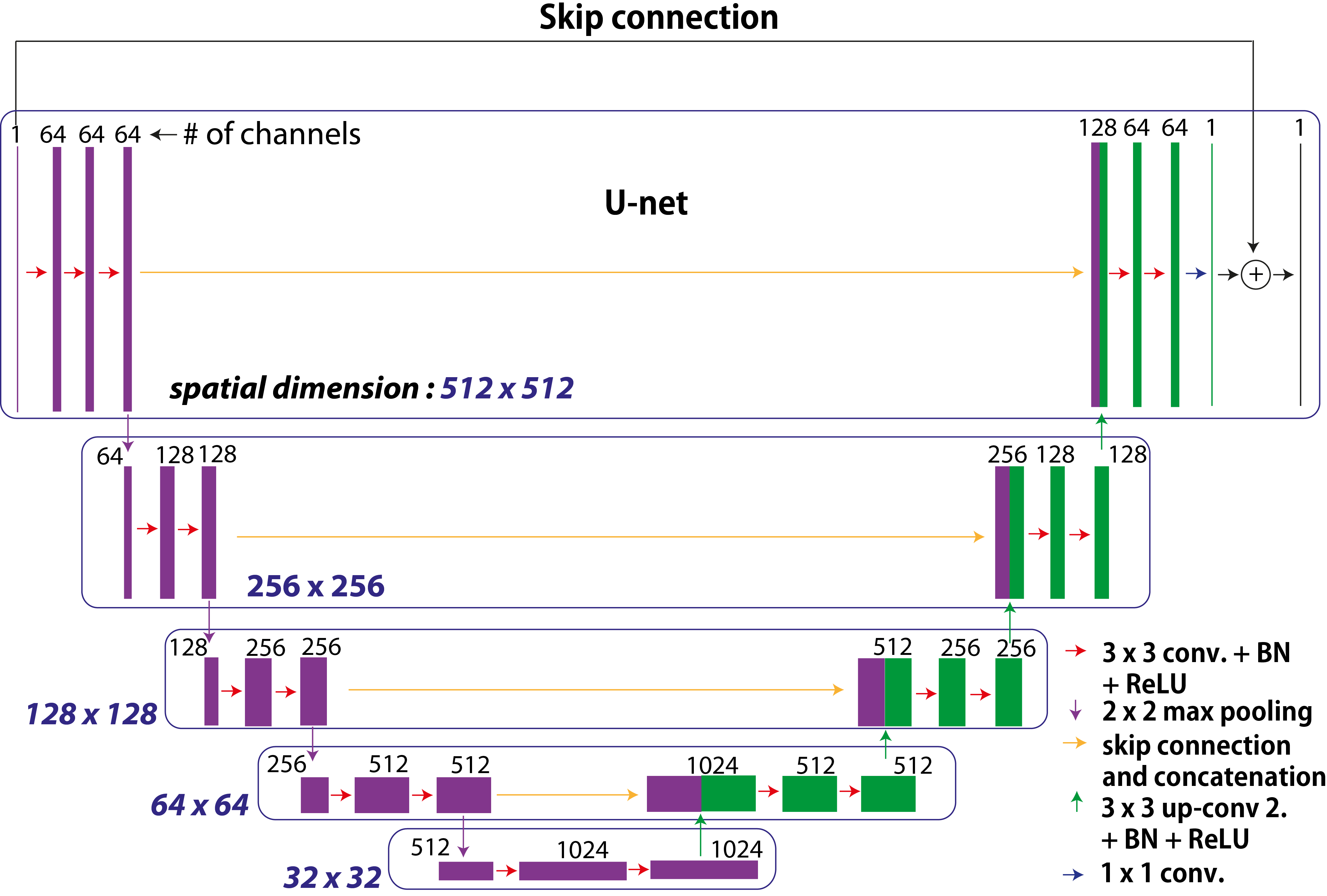}
\caption{Structure of the proposed deep convolutional network for deconvolution.}
\label{fig:whole_structure}
\end{figure*}

While we were inspired by the general form of the proximal update, \eqref{eq:prox3}, to apply a CNN to inverse problems of this form, our goal here is not to imitate iterative methods (e.g. by building a network that corresponds to an unrolled version of some iterative method), but rather to explore a state-of-the-art CNN architecture.
To this end, we base our CNN on the U-net \cite{ronneberger2015u}, which was originally designed for segmentation.
There are several properties of this architecture that recommend it for our purposes.

{\bf Multilevel decomposition.}
The U-net employs a dyadic scale decomposition based on max pooling, so that the effective filter size in the middle layers is larger than that of the early and late layers.
This is critical for our application because the filters corresponding to $H^*H$ (and its inverse) may have non-compact support, e.g. in CT.
Thus, a CNN with a small, fixed filter size may not be able to effectively invert $H^*H$.
This decomposition also has a nice analog to the use of multiresolution wavelets in iterative approaches.

{\bf Multichannel filtering.}
U-net employs multichannel filters, such that there are multiple feature maps at each layer.
This is the standard approach in CNNs~\cite{krizhevsky2012imagenet} to increase the expressive power of the network \cite{zeiler2014visualizing}.
The multiple channels also have an analog in iterative methods:
In the ISTA formulation \eqref{eq:prox3}, we can think of the wavelet coefficient vector $\ab$ as being partitioned into different channels, with each channel corresponding to one wavelet subband~\cite{da2006nonsubsampled,mallat1999wavelet}.
Or, in ADMM\cite{boyd2011distributed}, the split variables can be viewed as channels. 
The CNN architecture greatly generalizes this by allowing filters to make arbitrary combinations of filters.

{\bf Residual learning.}
As a refinement of the original U-net, we add a skip connection~\cite{kim2015accurate} between input and output, which means that the network actually learns the difference between input and output.
This approach mitigates the vanishing gradient problem \cite{he2015deep} during training.
This yields a noticeable increase in performance compared to the same network without the skip connection.

{\bf Implementation details.}
We made two additional modification to U-net.
First, we use zero-padding so that the image size does not decrease after each convolution.
Second, we replaced the last layer with a convolutional layer which reduces the 64 channels to a single output image.
This is necessary because the original U-net architecture results in two channgels: foreground and background.

\section{Experiments and results}
\label{sec:experiments}
We now describe our experimental setup and results.
Though the FBPConvNet algorithm is general, we focus here on sparse-view X-ray CT reconstruction.
We compare FBPConvNet to FBP alone and a state-of-the-art iterative reconstruction method~\cite{McCann:16}.
This method (which we will refer to as the \emph{TV} method for brevity) solves a version of Eq. \eqref{eq:cost2} with the popular TV regularization via ADMM.
It exploits the convolutional structure of $H^*H$ by using FFT-based filtering in its iterates.

Our experiments proceed as follows:
We begin with a full view sinogram (either synthetically generated or from real data).
We compute its FBP (standard high quality reconstruction) and take this as the ground truth.
We then compare the results of applying the TV method to the subsampled sinogram with the results of applying the FBPConvNet to the same.
This type of sparse-view reconstruction is of particular interest for human imaging because, e.g., a twenty times reduction in the number of views corresponds to a twenty times reduction in the radiation dose received by the patient.

\subsection{Data Preparation}
We used three datasets for evaluations of the proposed method.
The first two are synthetic in that the sinograms are computed using a digital forward model, while the last comes from real experiments.
\begin{enumerate}
    \item The \emph{ellipsoid dataset} is a synthetic dataset that comprises 500 images of ellipses of random intensity, size, and location.
    Sinograms for this data are 729 pixels by 1,000 views and are created using the analytical expression for the X-ray transform of an ellipse.
    The Matlab function \texttt{iradon} is used for FBPs.
    \item The \emph{biomedical dataset} is a synthetic dataset that comprises 500 real in-vivo CT images from the Low-dose Grand challenge competition from database made by the Mayo clinic.
   Sinograms for this data are 729 pixels by 1,000 views and are created using the Matlab function \texttt{radon}.
   \texttt{iradon} is again used for FBPs.
    \item The \emph{experimental dataset} is a real CT dataset that comprises 377 sinograms collected from an experiment at the TOMCAT beam line of the Swiss Light Source at the Paul Scherrer Institute in Villigen, Switzerland.
    Each sinogram is 1493 pixels by 721 views and comes from one z-slice of a single rat brain.
    FBPs were computed using our own custom routine which closely matches the behavior of \texttt{iradon} while accommodating different sampling steps in the sinogram an reconstruction domains.
\end{enumerate}

 To make sparse-view FBP images in synthetic datasets, we uniformly subsampled the sinogram by factors of 7 and 20 corresponding to 143 and 50 views, respectively. For the real data, we subsampled by factors of 5 and 14 corresponding to 145 and 52 views.

\subsection{Training Procedure}

{\bf FBPConvNet.}
In case of synthetic data, the total number of training images is 475. The number of test images is 25.
In the case of the biomedical dataset, the test data is chosen from a different subject than the training set.  For the real data, the total number of training images is 327. The number of test images is 25.
The test data are obtained from the last z-slices with the gap of 25 slices left between testing and training data.
All images are scaled between 0 and 550. 

The CNN part of the FBPConvNet is trained using pairs of low-view FBP images and full-view FBP images as input and output, respectively.
Note that this training strategy means that the method is applicable to real CT reconstructions where we do not have access to an oracle reconstruction.

We use the MatConvNet toolbox (ver. 20) \cite{vedaldi15matconvnet} to implement the FBPConvNet training and evaluation, with a slight modification:
We clip the computed gradients to a fixed range to prevent the divergence of the cost function \cite{kim2015accurate,pascanu2013difficulty}. We use a Titan Black GPU graphic processor (NVIDIA Corporation) for training and evaluation. Total training time is about 15 hours for 101 iterations.

The hyper parameters for training are as follows: learning rate decreasing logarithmically from 0.01 to  0.001; batchsize equals 1; momentum equals 0.99; and the clipping value for gradient equals $10^{-2}$.
We use flip-flop data augmentation in both horizontal and vertical directions during the training phase to reduce overfitting \cite{krizhevsky2012imagenet}.

{\bf State-of-the-art TV reconstruction.}
For completeness, we comment on how the iterative method used the training and testing data.
Though it may be a fairer comparison to require the TV method to select its parameters from the training data (as the FBPConvNet does), we instead simply select the parameters that optimize performance on the training set (via a golden-section search).
We do this with the understanding that the parameter is usually tuned by hand in practice and because the correct  way to learn these parameters from data remains an open question. 



\section{Experimental Results}\label{sec:result}
 We use SNR as a quantitative metric. If $\xb$ is the oracle and $\hat{\xb}$ is the reconstructed image, SNR is given by
\begin{equation}
    \mbox{SNR}=\max_{a,b\in\Rd} 20 \log \frac{\|\xb\|_2}{\|\xb-a\hat{\xb}+b\|_2},
\end{equation}
where a higher SNR value corresponds to a better reconstruction.

\subsection{Ellipsoidal Dataset}

\begin{figure*}[htbp]
\centering
\includegraphics[trim = 0mm 50mm 0mm 40mm,clip=true,width=16.5 cm]{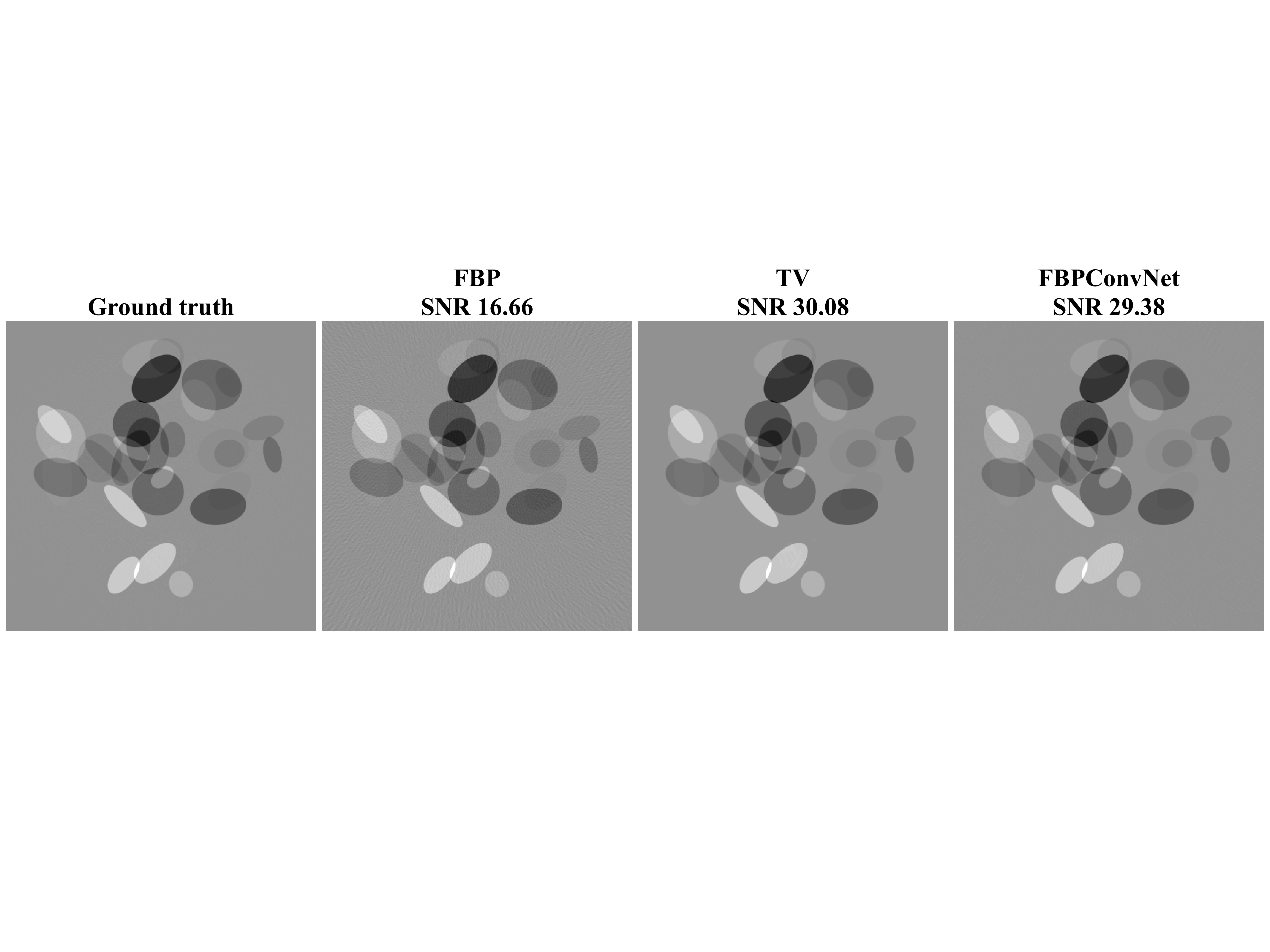}
\includegraphics[trim = 0mm 50mm 0mm 49.5mm,clip=true,width=16.5 cm]{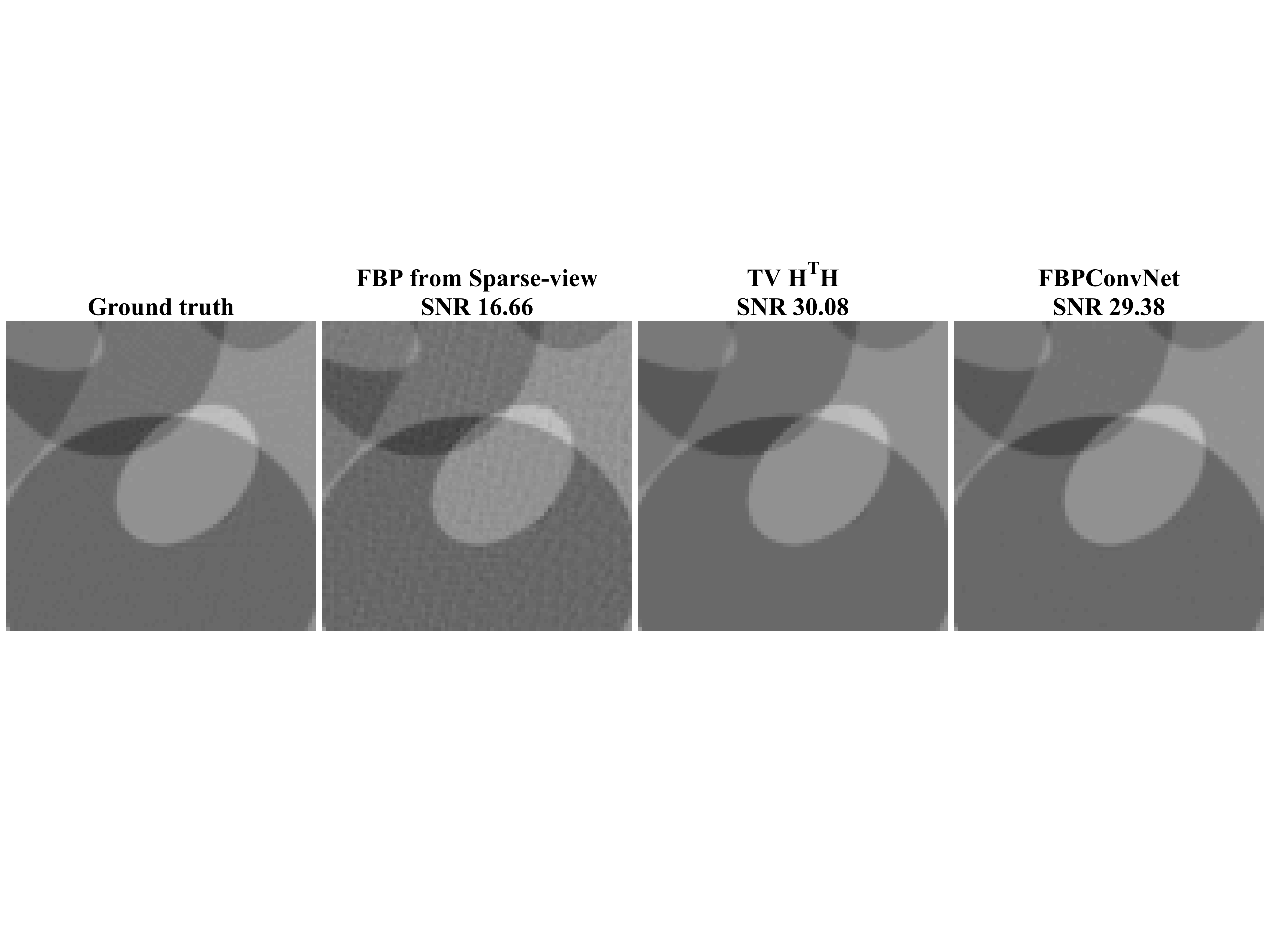}
\caption{Reconstructed images of ellipsoidal dataset from 143 views using FBP, TV regularized convex optimization \cite{McCann:16}, and the FBPConvNet. The first row shows the ROI with full region of image, and the second row shows magnified ROI for the appearance of differences. All subsequent figures keep the same manner.} 
\label{fig:res1_el}
\end{figure*}

\begin{figure*}[htbp]
\centering
\includegraphics[trim = 0mm 50mm 0mm 40mm,clip=true,width=16.5 cm]{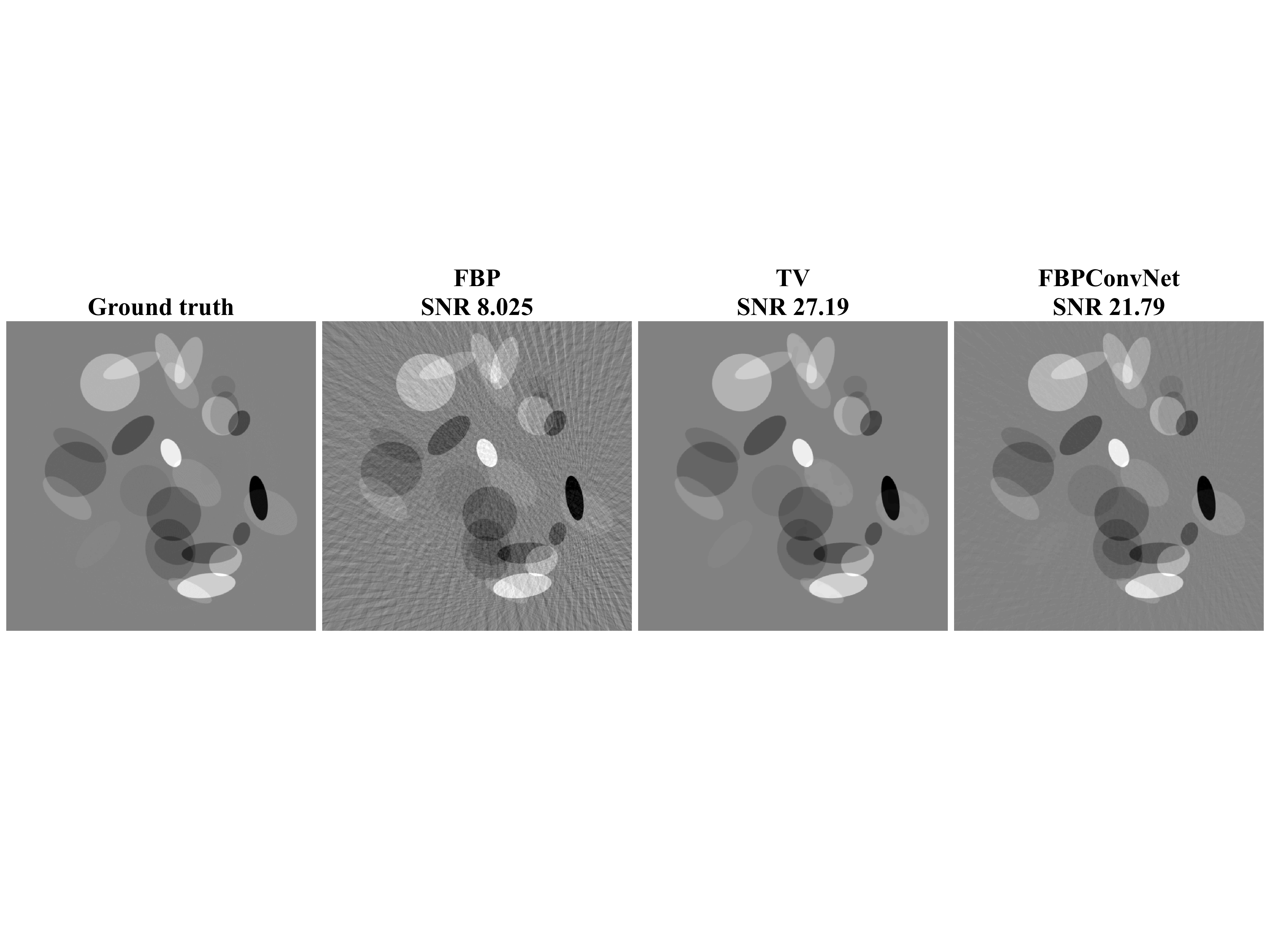}
\includegraphics[trim = 0mm 50mm 0mm 49.5mm,clip=true,width=16.5 cm]{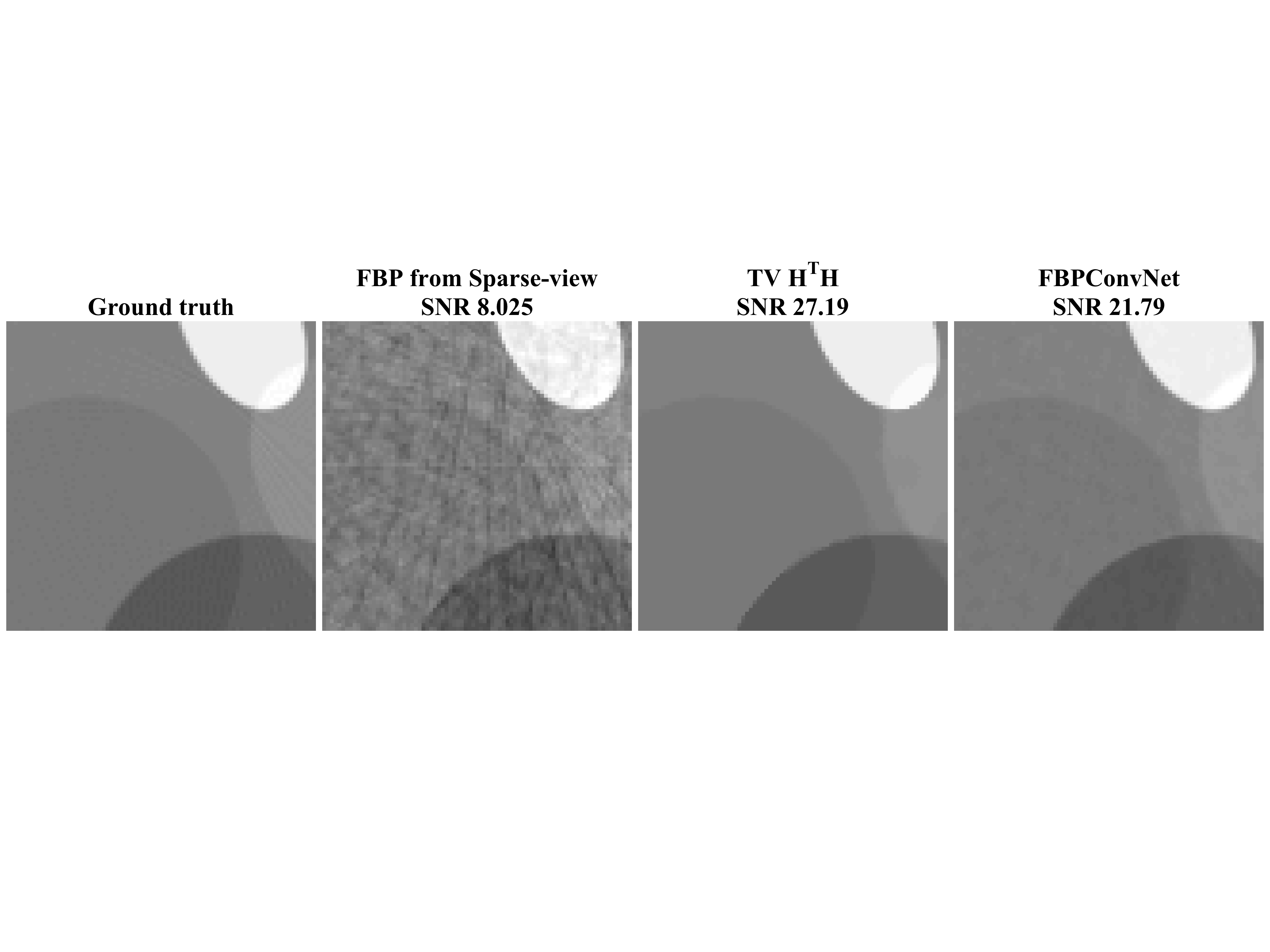}
\caption{Reconstructed images of ellipsoidal dataset from 50 views using FBP, TV regularized convex optimization \cite{McCann:16}, and the FBPConvNet.}
\label{fig:res2_el}
\end{figure*}

Figures~\ref{fig:res1_el} and \ref{fig:res2_el} and Table~\ref{tab:numerical} show the results for the ellisoidal dataset.
In the seven times downsampling case, Figure~\ref{fig:res1_el},
the full-view FBP (ground truth) shows nearly artifact-free ellipsoids, while the sparse-view FBP shows significant line artifacts (most visible in the background).
Both the TV and FBPConvNet methods significantly reduce these artifacts, giving visually indistinguishable results.
When the downsampling is increased to twenty times, Figure~\ref{fig:res2_el}, the line artifacts in the sparse-view FBP reconstruction are even more pronounced.
Both the TV and FBPConvNet reduce these artifacts, though the FBPConvNet retains some of the artifacts.
The average SNR on the testing set for the TV method is higher than that of the the FBPConvNet.
This is a reasonable results given that the phantom is piecewise constant and thus the TV regularization should be optimal~\cite{chambolle2004algorithm,unser2016representer}.

\begin{table}[htbp]
\caption{Comparison of SNR between different reconstruction algorithms for numerical ellipsoidal dataset.}\label{tab:el}
\renewcommand{\arraystretch}{1.2} \centering
\begin{tabular}{c|l||ccc}
  \hline
  \hline
 \multicolumn{2}{l||}{\diagbox{Metrics}{~~~~~~~~Methods}}	& FBP  	& TV \cite{McCann:16} & Proposed\\
\hline
\hline
  \multirow{2}{*}{avg. SNR (dB)}&   143 \mbox{views} (x7)& 16.09 	& 29.48 & 28.96 \\ 
\cline{2-5}
  &  50 \mbox{views} (x20)	 						& 8.44 	& 27.69 & 23.84\\
\hline
\hline
\end{tabular}\label{tab:numerical}
\end{table}

\subsection{Biomedical Dataset}

Figures~\ref{fig:res1_biomed} and \ref{fig:res2_biomed} and Table~\ref{tab:biomed} show the results for the biomedical dataset.
In Figure \ref{fig:res1_biomed}, again, the sparse-view FBP contains line artifacts.
Both TV and the proposed method remove streaking artifacts satisfactorily; however, the TV reconstruction shows the cartoon-like artifacts that are typical of TV reconstructions. This trend is also observed in severe case (x20) in Fig. \ref{fig:res2_biomed}.
Quantitatively, the proposed method outperforms the TV method.

\begin{figure*}[htbp]
\centering
\includegraphics[trim = 0mm 50mm 0mm 40mm,clip=true,width=16.5 cm]{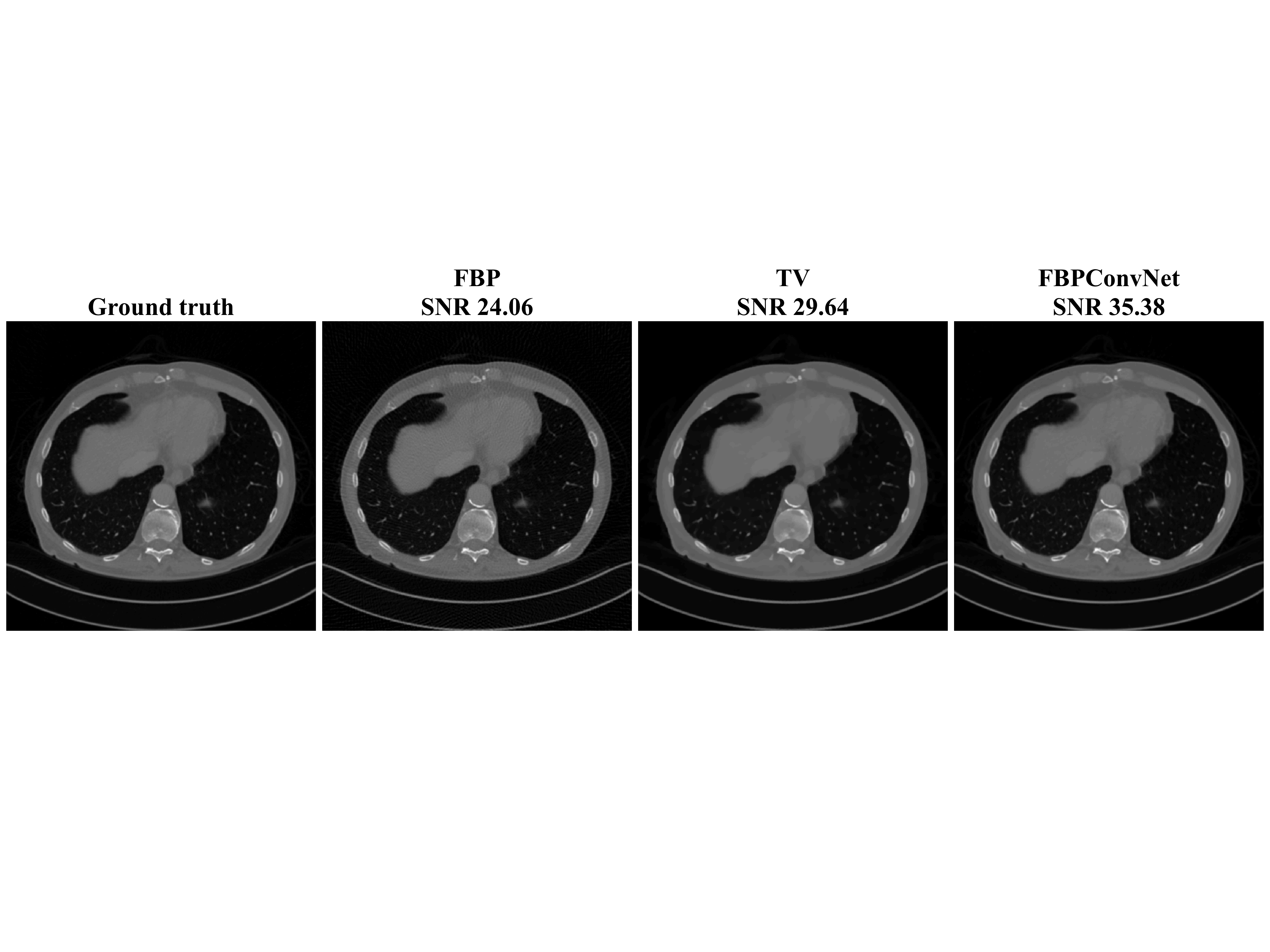}
\centering
\includegraphics[trim = 0mm 50mm 0mm 49.5mm,clip=true,width=16.5 cm]{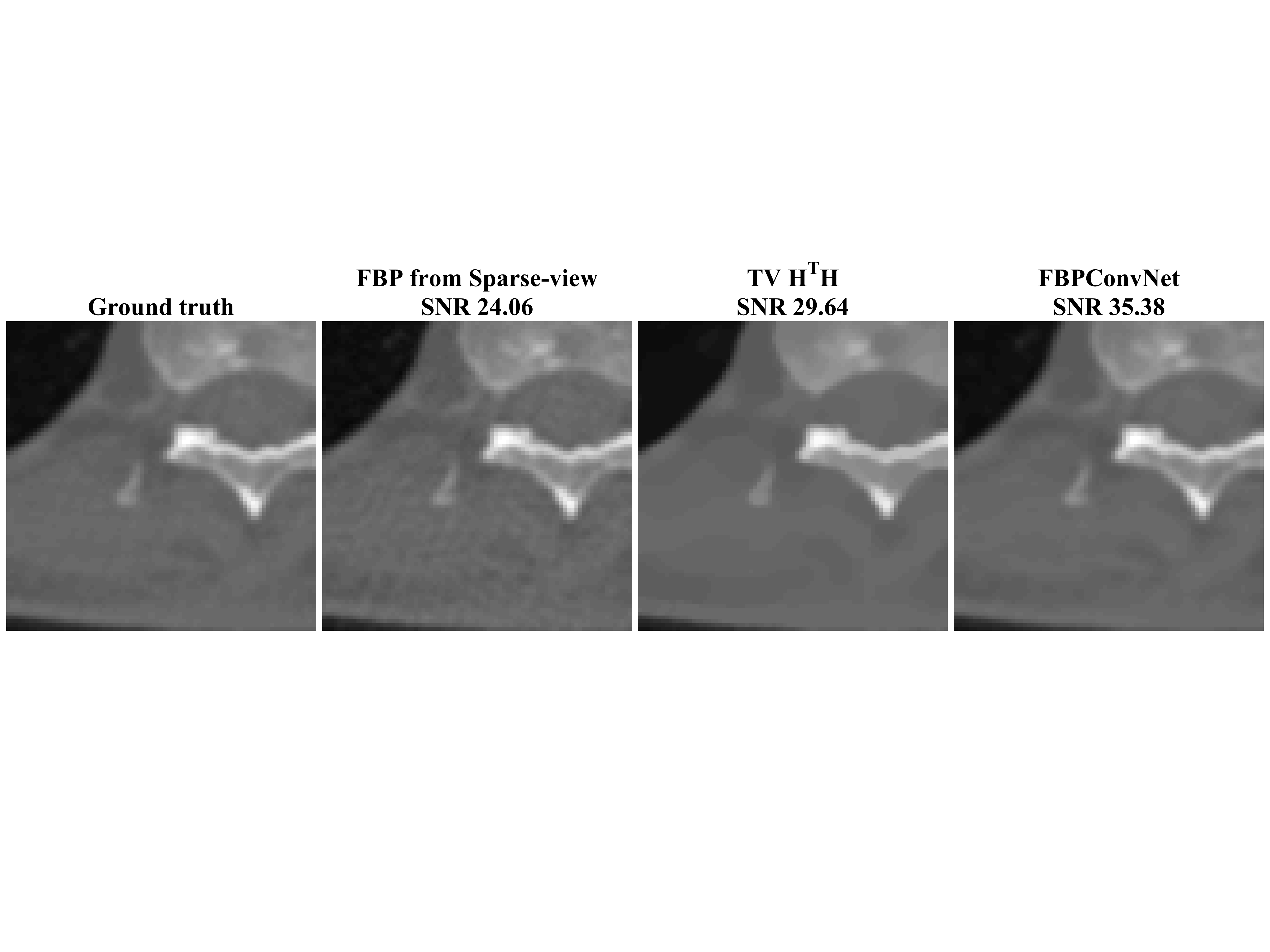}
\caption{Reconstructed images of biomedical dataset from 143 views using FBP, TV regularized convex optimization \cite{McCann:16}, and the FBPConvNet.}
\label{fig:res1_biomed}
\end{figure*}

\begin{figure*}[htbp]
\centering
\includegraphics[trim = 0mm 50mm 0mm 40mm,clip=true,width=16.5 cm]{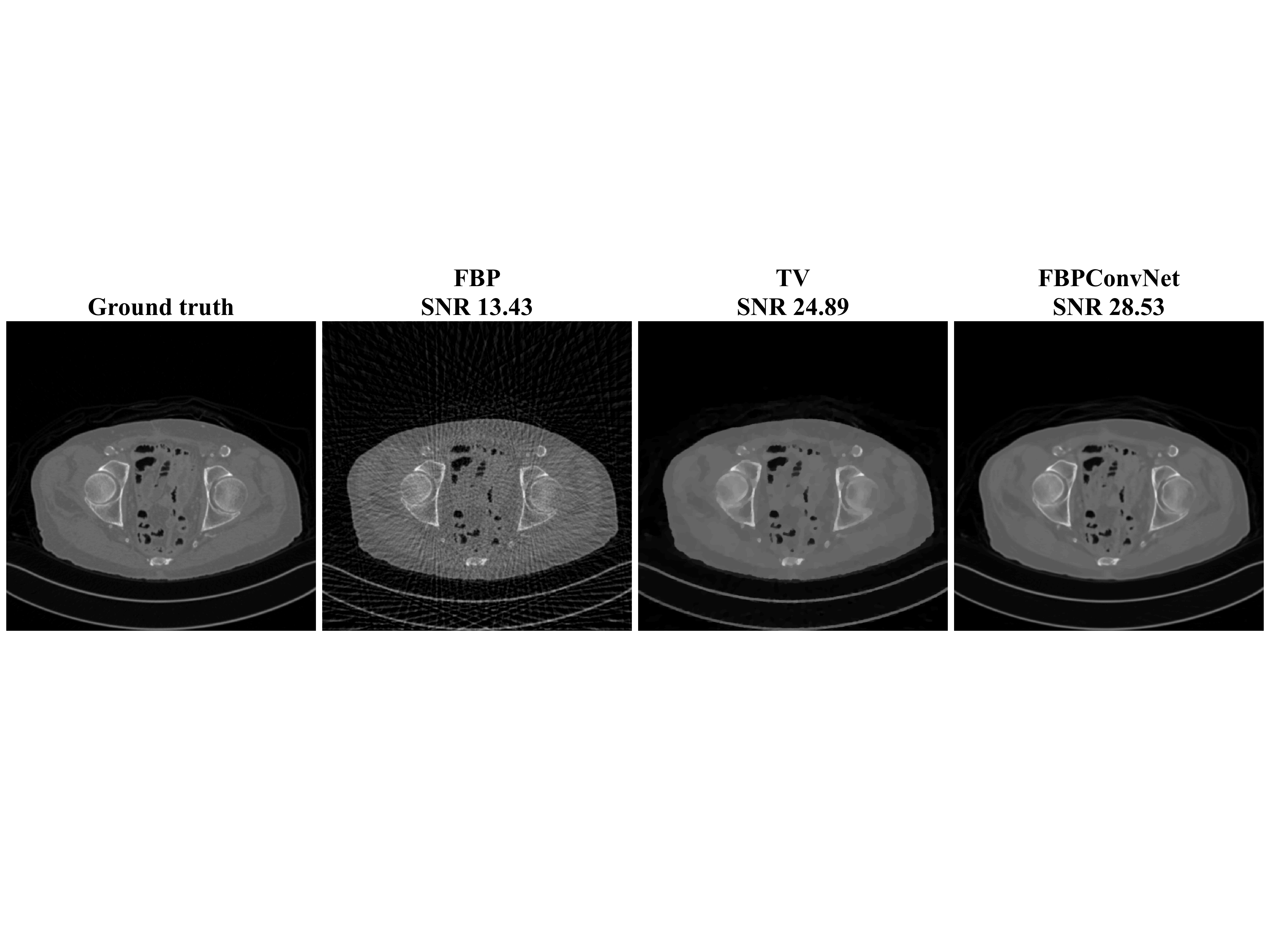}
\includegraphics[trim = 0mm 50mm 0mm 49.5mm,clip=true,width=16.5 cm]{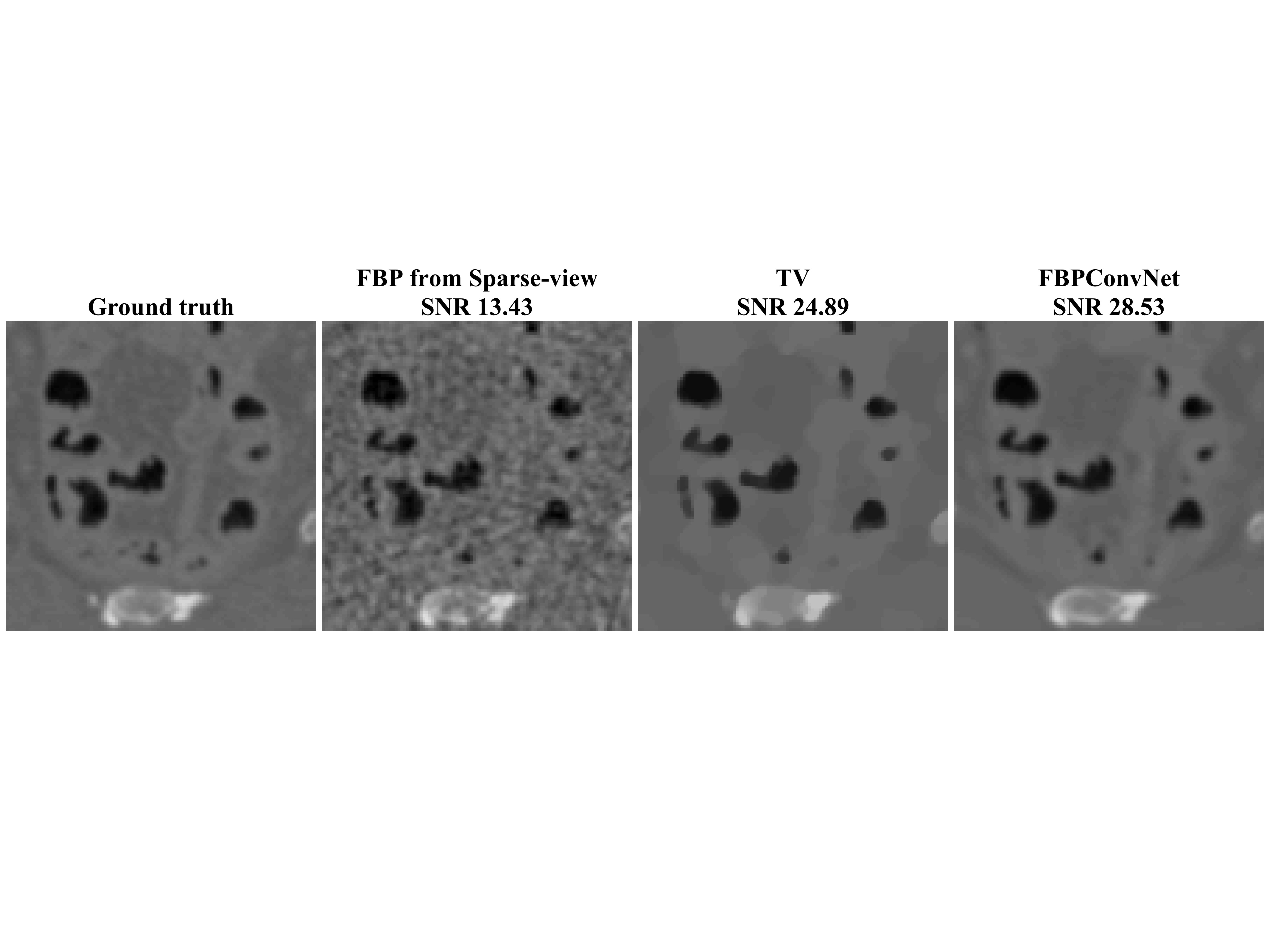}
\caption{Reconstructed images of biomedical dataset from 50 views using FBP, TV regularized convex optimization \cite{McCann:16}, and the FBPConvNet.}
\label{fig:res2_biomed}
\end{figure*}

\begin{table}[htbp]
\caption{Comparison of SNR between different reconstruction algorithms for biomedical dataset.}
\renewcommand{\arraystretch}{1.2} \centering
\begin{tabular}{c|l||ccc}
  \hline
  \hline
 \multicolumn{2}{l||}{\diagbox{Metrics}{~~~~~~~~Methods}}	& FBP  	& TV \cite{McCann:16} & Proposed\\
\hline
\hline
  \multirow{2}{*}{avg. SNR (dB)}&   143 \mbox{views} (x7)& 24.97 	& 31.92 & 36.15 \\ 
\cline{2-5}
  &  50 \mbox{views} (x20)	 						& 13.52 	& 25.2 & 28.83\\
\hline
\hline
\end{tabular}\label{tab:biomed}
\end{table}

\subsection{Experimental Dataset}

\begin{figure*}[htbp]
\centering
\includegraphics[trim = 0mm 50mm 0mm 40mm,clip=true,width=16.5 cm]{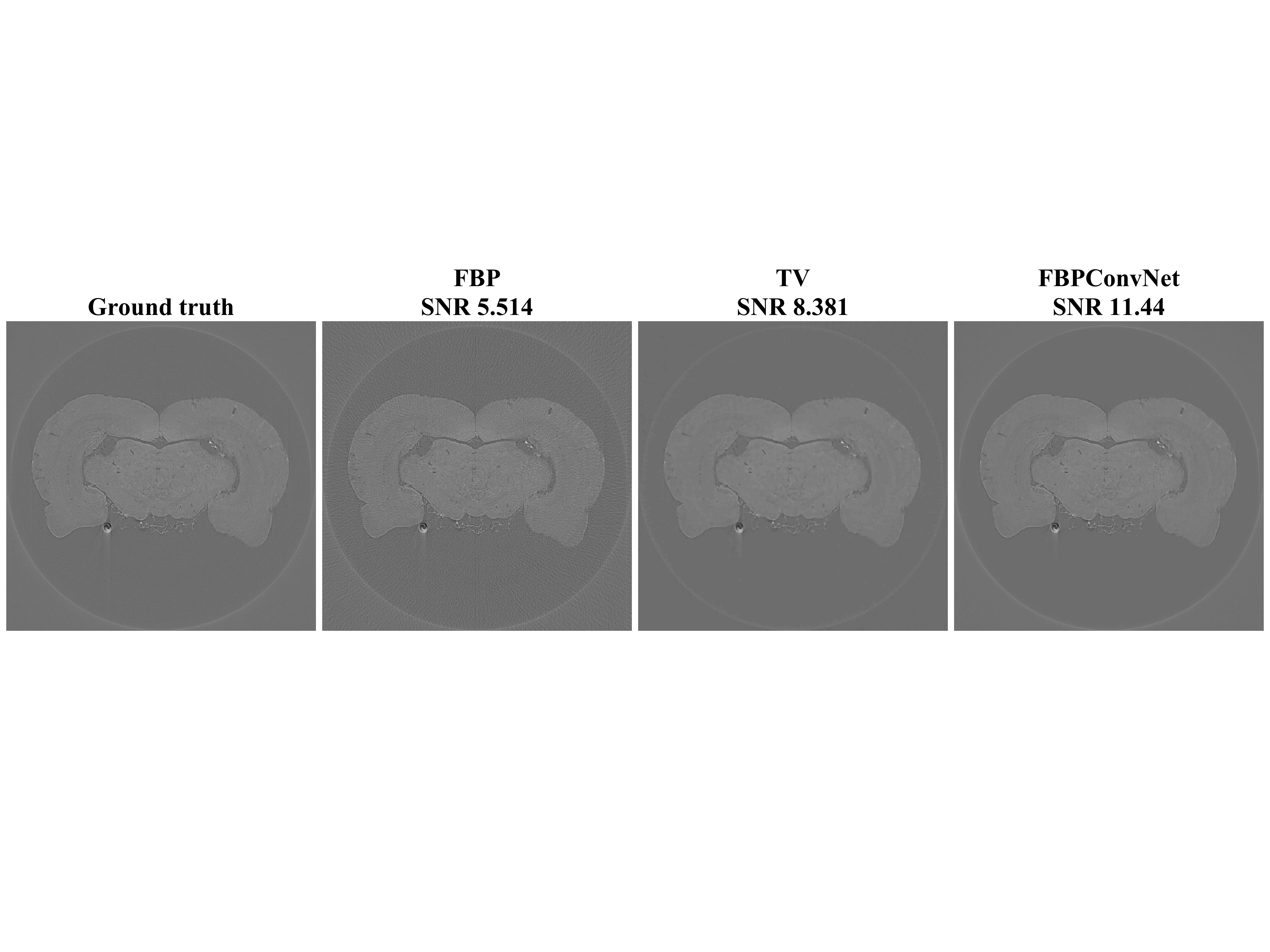}
\includegraphics[trim = 0mm 50mm 0mm 49.5mm,clip=true,width=16.5 cm]{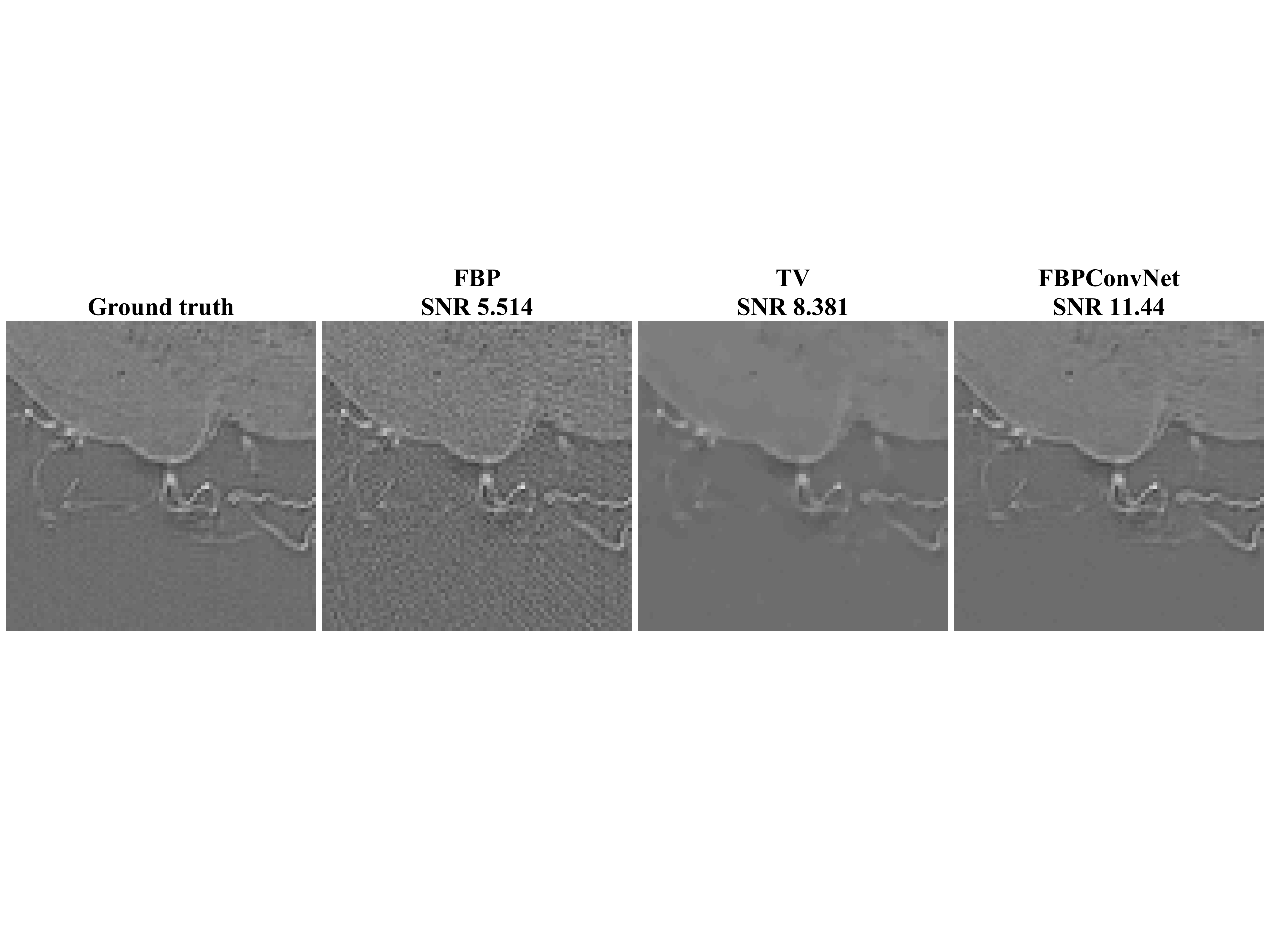}
\caption{Reconstructed images of experimental dataset from 145 views using FBP, TV regularized convex optimization \cite{McCann:16}, and the FBPConvNet.}
\label{fig:res1_rat}
\end{figure*}

\begin{figure*}[htbp]
\centering
\includegraphics[trim = 0mm 50mm 0mm 40mm,clip=true,width=16.5 cm]{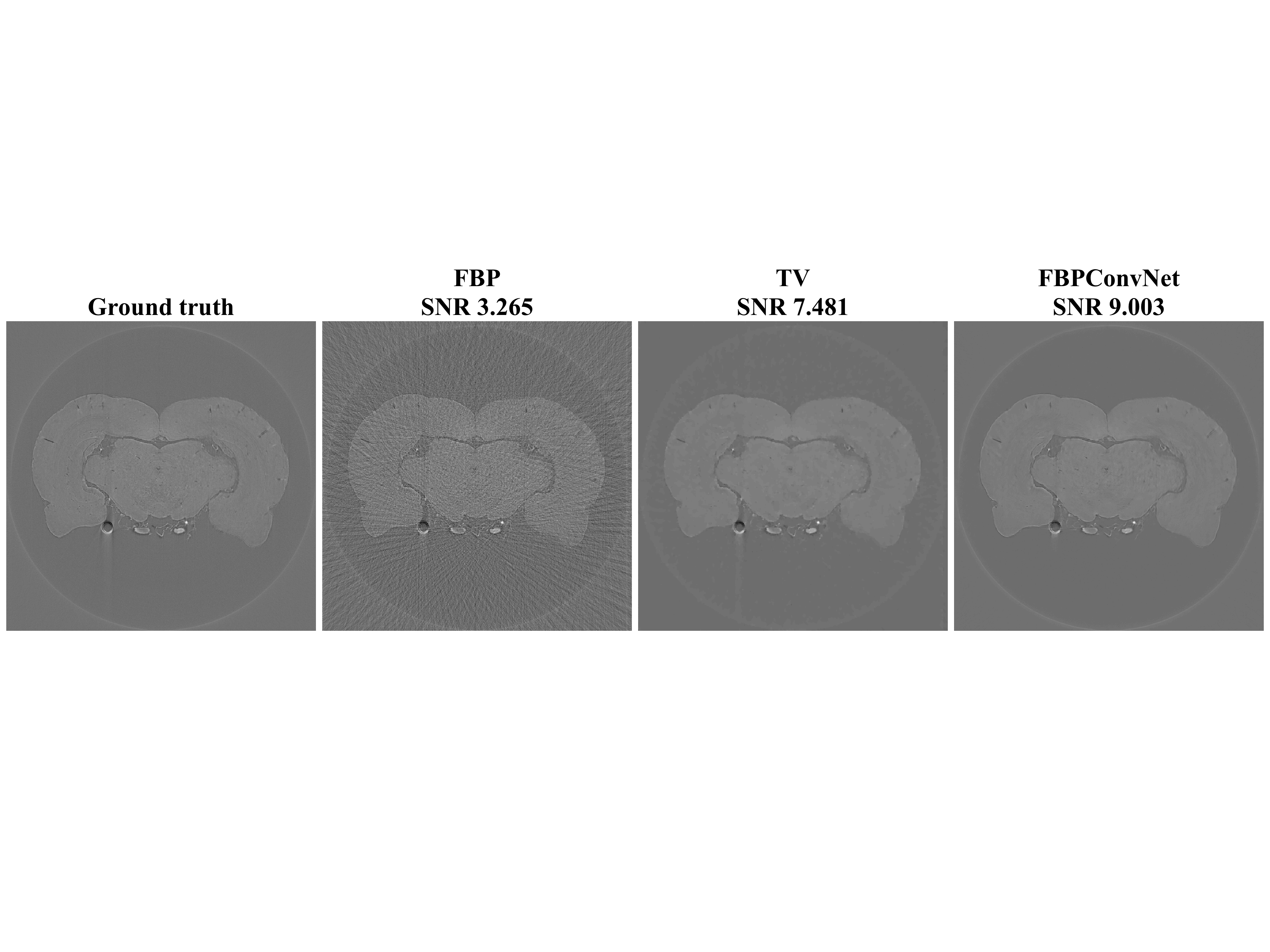}
\includegraphics[trim = 0mm 50mm 0mm 49.5mm,clip=true,width=16.5 cm]{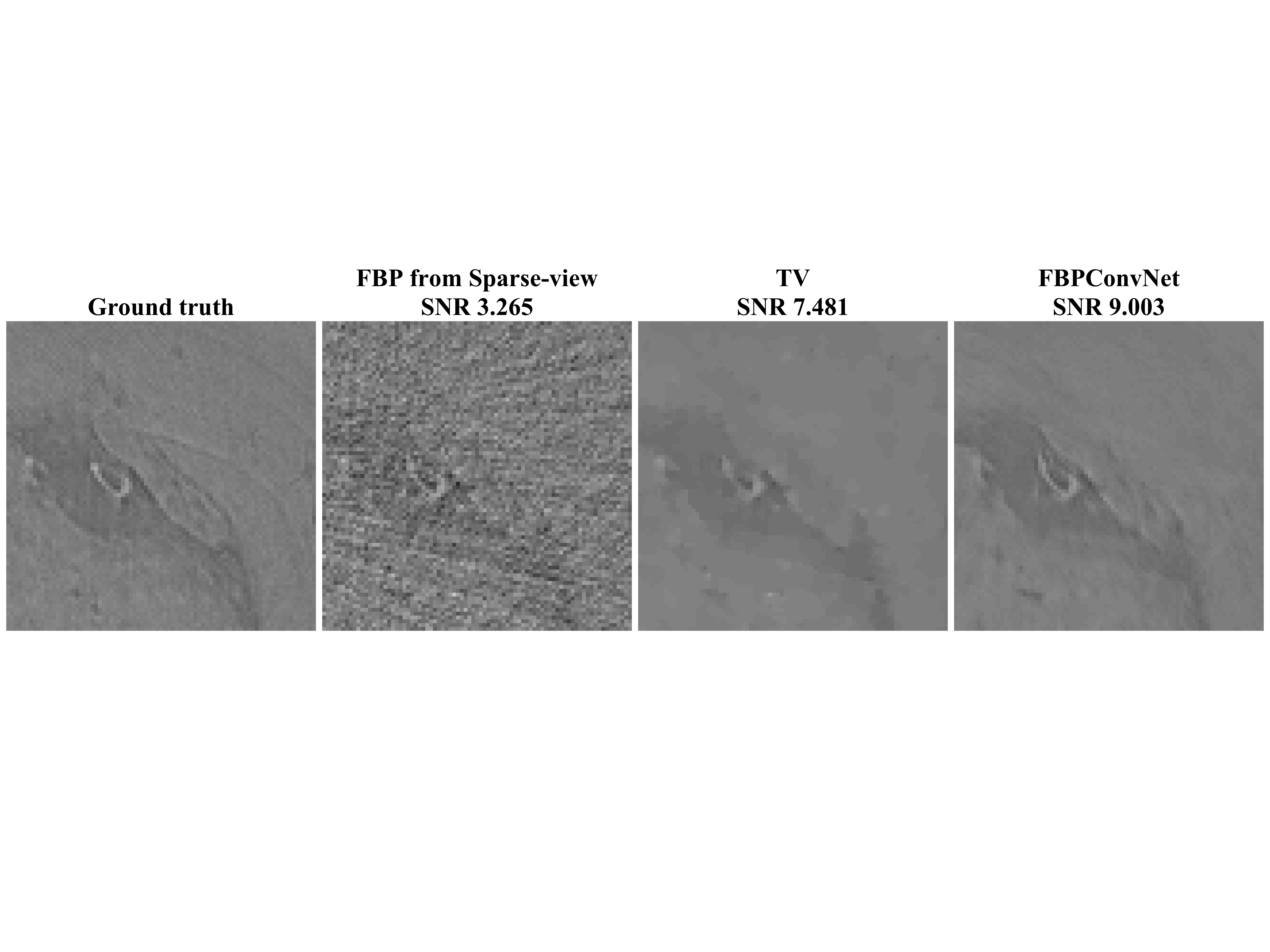}
\caption{Reconstructed images of experimental dataset from 52 views using FBP, TV regularized convex optimization \cite{McCann:16}, and the FBPConvNet.}
\label{fig:res2_rat}
\end{figure*}

\begin{table}[htbp]
\caption{Comparison of SNR between different reconstruction algorithms for  experimental dataset.}
\renewcommand{\arraystretch}{1.2} \centering
\begin{tabular}{c|l||ccc}
  \hline
  \hline
 \multicolumn{2}{l||}{\diagbox{Metrics}{~~~~~~~~Methods}}	& FBP  	& TV \cite{McCann:16} & Proposed\\
\hline
\hline
  \multirow{2}{*}{avg. SNR (dB)}&   145 \mbox{views} (x5)& 5.38 	& 8.25 & 11.34 \\ 
\cline{2-5}
  &  51 \mbox{views} (x14)	 						& 3.29 	& 7.25 & 8.85\\
\hline
\hline
\end{tabular}\label{tab:real}
\end{table}

Figures~\ref{fig:res1_rat} and \ref{fig:res2_rat} and Table~\ref{tab:real} show the results for the experimental dataset.
The SNRs of all methods are significantly lower here because of the relatively low contrast of the sinogram.
In Fig. \ref{fig:res1_rat}, we observe the same trend as for the biomedical dataset, where the TV method oversmooths and the FBPConvNet better preserves fine structures.
These trends also appears in twenty times downsampling case (x20) in Fig. \ref{fig:res2_rat}.
The FBPConvNet had a higher SNR than the TV method in both settings.

\section{Discussion}\label{sec:discussion}
The experiments provide strong evidence for the feasibility of the FBPConvNet for sparse-view CT reconstruction.
The conventional iterative algorithm with TV regularization outperformed the FBPConvNet in the ellipsoidal dataset, while the reverse was true for the biomedical and experimental datasets.
In these more-realistic datasets, the SNR improvement of the FBPConvNet came from its ability to preserve fine details in the images.
This points to one advantage of the proposed method over iterative methods: the iterative methods must explicitly impose regularization, while the FBPConvNet effectively learns a regularizer from the data.

The computation time for the FBPConvNet was about 200 ms for the FBP and 200$\sim$300 ms in GPU for the CNN for a $512\times512$ image. This is much faster than the iterative reconstruction, which, in our case, requires around 7 minutes even after the regularization parameters have been selected.

A major limitation of the proposed method is lack of transfer between datasets. For instance, when we put FBP images from a twenty-times subsampled sinogram into the network trained on the seven-times subsampled sinogram, the results retain many artifacts. Handling datasets of different dimensions or subsampling factors requires retraining the network.
Future work could address strategies for heterogeneous datasets.

Our theory suggests that the methodology proposed here is applicable to all problems where the normal operator is shift-invariant; however, we have focused here on X-ray CT reconstruction.
We expect that adapting the method to, e.g., MRI reconstruction will be non-trivial experimentally, because it will require large sets of training data (either from a high-quality forward model or real data) and a high-quality iterative reconstruction algorithm for comparison.
Furthermore, because MRI and DT involve complex values (both in the measurement and reconstruction domains), we need a CNN architecture that correctly handles complex values. Therefore, we leave experimentation on other modalities to future work.

\section{Conclusion}\label{sec:conclusion}
In this paper, we proposed a deep convolutional network for inverse problems with a focus on biomedical imaging. The proposed method, which we call the \emph{FBPConvNet} combines FBP with a multiresolution CNN.
The structure of the CNN is based on U-net, with the addition of residual learning.

This approach was motivated by the convolutional structure of several biomedical inverse problems, including CT, MRI, and DT.
Specifically, we showed conditions on a linear operator that ensure that its normal operator is a convolution.
This results suggests that CNNs are well-suited to this subclass of inverse problems.

The proposed method demonstrated compelling results on synthetic and real data. 
It compared favorably to state-of-the-art iterative reconstruction on the two more realistic datasets.
Furthermore, after training, the computation time of the proposed network per one image is under a second.

\section*{Acknowledgment}
The authors would like to thank Dr. Cynthia McCollough, the Mayo Clinic, the American Association of Physicists in Medicine, and grants EB017095 and EB017185 from the National Institute of Biomedical Imaging and Bioengineering for giving opportunities to use real-invivo CT DICOM images (Fig. \ref{fig:res1_biomed}-\ref{fig:res2_biomed}). The authors also thank thank Dr. Marco Stampanoni, Swiss Light Source, Paul Scherrer Institute, Villigen, Switzerland, for providing real CT sinograms (Fig. \ref{fig:res1_rat}-\ref{fig:res2_rat}). 

\ifCLASSOPTIONcaptionsoff
  \newpage
\fi



\bibliographystyle{IEEEtran}
\clearpage

\end{document}